\documentclass[10pt,twocolumn,letterpaper]{article}

\usepackage[pagenumbers]{wacv} % To force page numbers, e.g. for an arXiv version

\usepackage{adjustbox}
\usepackage{multirow}

\usepackage{algorithm}
\usepackage{algorithmic}

\usepackage{amsthm}
\usepackage{amsmath}

\newtheorem{theorem}{Theorem}
\newtheorem{lemma}[theorem]{Lemma}

\definecolor{wacvblue}{rgb}{0.21,0.49,0.74}
\usepackage[pagebackref,breaklinks,colorlinks,allcolors=wacvblue]{hyperref}

\def\wacvPaperID{565} % *** Enter the WACV Paper ID here
\def\confName{WACV}
\def\confYear{2026}

\title{Data-Driven Lipschitz Continuity: \\ A Cost-Effective Approach to Improve Adversarial Robustness}

\author{Erh-Chung Chen \textsuperscript{1}
\and
Pin-Yu Chen \textsuperscript{2}
\and
I-Hsin Chung \textsuperscript{2}
\and 
Che-Rung Lee \textsuperscript{1} 
\and
National Tsing Hua University \textsuperscript{1} \\
\and
IBM Research \textsuperscript{2} \\
}

\begin{document}
\maketitle
\begin{abstract}
As deep neural networks (DNNs) are increasingly deployed in sensitive applications, ensuring their security and robustness has become critical. A major threat to DNNs arises from adversarial attacks, where small input perturbations can lead to incorrect predictions. Recent advances in adversarial training improve robustness by incorporating additional examples from external datasets or generative models. However, these methods often incur high computational costs, limiting their practicality and hindering real-world deployment.
In this paper, we propose a cost-efficient alternative based on Lipschitz continuity that achieves robustness comparable to models trained with extensive supplementary data. Unlike conventional adversarial training, our method requires only a single pass over the dataset without gradient estimation, making it highly efficient. Furthermore, our method can integrate seamlessly with existing adversarial training frameworks and enhances the robustness of models without requiring extra generative data. Experimental results show that our approach not only reduces computational overhead but also maintains or improves the defensive capabilities of robust neural networks. This work opens a promising direction for developing practical, scalable defenses against adversarial attacks.
\end{abstract}

%%%%% %%%%% %%%%% %%%%% %%%%% %%%%% %%%%% %%%%% %%%%% %%%%% %%%%% %%%%%
\section{Introduction}
Deep neural networks (DNNs) have achieved remarkable success across a wide range of tasks \citep{krizhevsky2012imagenet,redmon2016you}, driving their adoption in increasingly sensitive applications. As a result, concerns over the security and robustness of these models have become more urgent, as even a single erroneous prediction can have serious consequences. For example, the Overload attack can significantly inflate the inference time of detecting objects for edge devices \citep{chen2023overload}, and even minor typos in prompts can cause large language models to behave unpredictably \citep{zhu2023promptbench}. Among the most critical threats are adversarial attacks, which add imperceptible perturbations to inputs, leading DNNs to make incorrect predictions. These vulnerabilities have been demonstrated across various domains \citep{szegedy2013intriguing,carlini2018audio,li2018textbugger}, highlighting the urgent need for more robust and explainable AI models.

Adversarial training \citep{madry2017towards} has proven to be an effective strategy for enhancing the robustness of DNNs by generating adversarial examples during training and optimizing model weights to minimize the losses caused by these examples. Despite promising results, no certified models have been deployed commercially due to prohibitive computational requirements. Recent studies have shown that robustness can be further improved by introducing additional examples from other datasets \citep{carmon2019unlabeled} or using generative models \citep{gowal2021improving, wang2023better} to cover low-frequency data. However, these approaches require datasets 100 or more times larger than the original, which significantly escalates computational demands by orders of magnitude. This creates a critical trade-off between training cost and robustness, presenting a significant obstacle to deploying robust DNN-based applications commercially are at stake.

The focus of this paper is to explore an alternative approach that reduces training cost while achieving robustness comparable to state-of-the-art adversarial training frameworks. To address this issue, we first revisit the theorem of Lipschitz continuity, which gauges how much outputs are amplified by perturbations. Unlike certified training \citep{mao2024understanding} or Lipschitz-constrained methods \citep{zuhlke2024adversarial} approaches that primarily aim to tighten upper bounds but often face scalability issues with larger models and datasets, our approach takes a fundamentally different direction. We prove that remapping the input domain of any layer that can be represented as a linear system to a constrained range results in a bounded Lipschitz constant and robustness improvements when specific conditions are met.

The key insight is that the forged function that reshapes the input domain of a linear system should maintain high similarity between the output of the transformed layers and the original, according to specific metrics. We propose a data-driven algorithm that constructs the optimal function with only a single scan over the dataset and no need for gradient estimation, making it highly efficient. This approach stands in contrast to prior methods that typically reshape weight distributions through regularization or by introducing complex objectives, which increases the computational cost of seeking suitable hyper-parameters.

Our key contributions are outlined as follows:
\begin{itemize}
    \item \textbf{Cost-Effective Robustness Enhancement:} We propose a method that enhance the robustness of existing adversarially trained methods with minimal additional costs. Compared to previous works, where improved robustness relied on introducing more generative data during the adversarial training phase, our method can be combined with existing models without re-training or fine-tuning. From a different perspective, for models with similar robustness, those integrated with our approach require significantly lower training cost.
    \item \textbf{Theoretical Foundation:} We provide a theoretical proof that, under specific conditions, remapping the input domain of any linear system to a constrained range leads to a bounded Lipschitz constant and improved robustness. This theoretical foundation distinguishes our work from previous Lipschitz-based approaches that focus primarily on constraining the network during training.
    \item \textbf{Inference-Time Optimization:} The proposed method is well-suited for inference-time model optimization and is almost cost-free. The function introduces only one parameter, the value of which can be determined by scanning observed data once without requiring gradient computations or fine-tuning. This makes our approach particularly valuable for practical applications where computational resources are limited.
    \item \textbf{State-of-the-Art Results:} Experimental results demonstrate that our method can be combined with various existing methods to achieve meaningful robustness improvements. Notably, models such as LTD \cite{chen2021ltd} and DefEAT \cite{chen2024data}, when combined with our method, achieve performance on par with RST-AWP \cite{wu2020adversarial}, a method that relies on extra data during training, against against adversarial examples generated by AutoAttack \citep{croce2020reliable}, a state-of-the-art ensemble attack, for CIFAR10, CIFAR100, and ImageNet datasets on the RobustBench leaderboard \citep{croce2020robustbench}. This finding supports the hypothesis that certain data samples may be safely omitted without sacrificing robustness, reinforcing the potential of developing data selection algorithms for adversarial training to reduce training cost without compromising effectiveness.
\end{itemize}

The rest of this paper is organized as follows. Section \ref{sec:bkg} introduces the background on adversarial attacks and adversarial training. Section \ref{sec:method} presents the theoretical proof of how robustness is enhanced by manipulating the domain of linear functions and introduces the proposed algorithm. Section \ref{sec:exp} shows the experimental results, including the comparisons among related works. Due to page limitations, ablation studies on various hyper-parameters, combination with different activation functions and gradient masking verification are left in Appendix. The last section is our conclusion.

%%%%% %%%%% %%%%% %%%%% %%%%% %%%%% %%%%% %%%%% %%%%% %%%%% %%%%% %%%%%
\section{Related Works}
\label{sec:bkg}
%-------------------------------------------------------------------------
\subsection{Adversarial Attacks}
% backdoor attack or poisoning attack
Adversarial attacks aim to inject tiny perturbations into inputs, causing victim DNNs to output incorrect predictions with high confidence \citep{chen2022adversarial}. These attacks have been observed in numerous vision applications \citep{goodfellow2014explaining,chen2018ead,wang2022di,yin2022adc}. Furthermore, these tiny perturbations can be embedded not only in image pixels but also in textual contexts \citep{kumar2023certifying, yao2023llm}, audio space \citep{xie2021enabling}, and other fields \citep{ilahi2021challenges}. Some research has shown how adversarial attacks threaten real applications \citep{xu2020adversarial,komkov2021advhat, du2022physical, wei2022adversarial}. Investigating the vulnerability of DNNs theoretically and designing robust DNNs is an ongoing challenge.

Adversarial attacks can be classified into two types based on the amount of information the attacker has access to: white-box attacks and black-box attacks. In the case of white-box attacks, the attacker has full access to all information about the victim model. Methods such as PGD attack \citep{madry2017towards} and AutoAttack \citep{croce2020reliable} generate adversarial examples by leveraging the gradient direction. Although this scenario is often unrealistic in practical settings, research in this area is valuable for developing more robust models in the future.

In contrast, black-box attacks assess the risks of adversarial attacks in real-world scenarios where attackers have limited access to the model's internals. Square Attack \citep{andriushchenko2020square} and ZO-NGD \citep{zhao2020towards} demonstrate that adversarial examples can be generated by querying the model's output predictions. Additionally, adversarial examples can be crafted through transferability, where models with similar architectures are used to generate adversarial examples that remain effective across different models \citep{wang2021admix, chen2024steal}.

%-------------------------------------------------------------------------
\subsection{Defensive Strategies}
Adversarial training is a defensive strategy that aims to find optimal weights against adversarial attacks. It achieves this by generating adversarial examples on the fly during the training phase and optimizing the model's weights to minimize the losses caused by these examples. Despite the superior robustness achieved by adversarial training, the associated training costs of adversarially trained models are generally ten times more expensive than those of models trained utilizing a standard policy. The concern over high computational costs becomes a significant obstacle in deploying DNN-based applications.

Balancing between training cost and robustness is a challenge for adversarial training. Fast adversarial training has been proposed for applications pursuing higher robustness under a limited budget \citep{chen2020towards, zhang2022revisiting}. However, numerous adversarial examples cannot be drawn from these approaches, potentially leading to catastrophic overfitting, where robust accuracy significantly decreases without warning signs \citep{rice2020overfitting}. On the contrary, some studies attempted to refine robustness by introducing additional examples from other datasets \citep{carmon2019unlabeled} or using generative models \citep{gowal2021improving, wang2023better}. Alternatively, another line of research has demonstrated that the removal of partial adversarial examples does not compromise robust accuracy, addressing the issue of unaffordable training costs \citep{zhang2020attacks, chen2024data}.

Despite the potential of adversarial training to enhance model robustness, budgetary constraints often limit the scope of their crafting to one or two specific attack types during the training stage. This restricted approach may inadvertently render adversarially trained models susceptible to novel, unseen attacks. As an alternative, Lipschitz-based certified training offers a theoretical framework for ensuring an upper bound on prediction errors \citep{gowal2018effectiveness,huang2021training,muller2022certified}. However, it is important to acknowledge that these training methods often suffer from scalability issues.

%%%%% %%%%% %%%%% %%%%% %%%%% %%%%% %%%%% %%%%% %%%%% %%%%% %%%%% %%%%%
\section{Methodology}
\label{sec:method}
%-------------------------------------------------------------------------
%\subsection{Motivation}
The main goal of this paper is to enhance the robustness of adversarially trained models by leveraging Lipschitz continuity theory to remap the input domain. We propose a data-driven approach that aachieves state-of-the-art performance with minimal extra computational cost, effectively balancing model resilience and training efficiency.

%-------------------------------------------------------------------------
\subsection{Lipschitz Continuity}
To achieve our goal, we introduce a quantitative metric known as the Lipschitz constant, which gauges how much outputs are amplified by the perturbations within the input domain. The mathematical definition is as follows, a function $f: \mathbb{R}^m \rightarrow \mathbb{R}^n$ is globally Lipschitz continuous if there exists an constant $k \geq 0$ such that
\begin{equation}
\label{eq:def_lip}
    D_f(f(x_1), f(x_2)) \leq k D_x( x_1, x_2) \quad \forall x_1, x_2 \in \mathbb{R}^m,
\end{equation}
where $D_x$ is a metric on the domain of $f$; $D_f$ is a metric on the range of $f$; and $x_1 \neq x_2$. For a DNN, it can be considered as a composite function:
\begin{equation}
    F(x) = (f_1 \circ f_2 \circ \dots \circ f_L)(x),
\end{equation}
where $f_i$ is the function of $i$-th layer. If there exists a Lipschitz constant for each individual layer, we can derive an upper bound of the Lipschitz constant for the victim model as follows,
\begin{equation}
\label{eq:Lip_F}
    k_{F} \leq \prod_{i=1}^{L} k_{i},
\end{equation}
where $k_i$ is the Lipschitz constant of $f_i$.

By defining adversarial examples $x^{\text{adv}}$ within a $\epsilon$-ball centered at an image $x$ as the inputs of (\ref{eq:def_lip}), we can assess the impact caused by adversarial examples. Therefore, the Lipschitz constant serves as a bridge that connects the design of robust models with the measurement of risks posed by adversarial examples. A small Lipschitz constant for the victim model implies that the increase in loss is minimal, indicating a higher ability to resist adversarial attacks. Consequently, the objective of this paper is to lower the upper bound of Lipschitz constant for the given models.

As indicated by previous studies \citep{yoshida2017spectral,farnia2018generalizable}, Lipschitz constant of the given model defined in (\ref{eq:Lip_F}) can be minimized by reducing the output discrepancy of individual linear layers. Under the $L_2$ norm, we have
\begin{align}
    \frac{|| f(x^{\text{adv}}) - f(x)||_2}{||x^{\text{adv}} - x||_2} &= \frac{(||Wx^{\text{adv}} + b) - (Wx + b)||_2}{||\delta||_2} \nonumber \\
    &= \frac{|| W\delta ||_2}{||\delta||_2},
\end{align}
where $W$ is the weight matrix; and $\delta$ is the distance between $x^{\text{adv}}$ and $x$. Therefore, the original optimization problem of minimizing Lipschitz constant is transformed into the following minimization problem:
\begin{equation}
\label{eq:min_eigen}
    \min_W \max_{\delta \neq 0, \delta \in \mathbb{R}^m} \frac{|| W\delta ||_2}{||\delta||_2} \quad \text{or}\quad \min_W \sigma_{\text{max}}(W),
    %\min_W \frac{|| W\delta ||_2}{||\delta||_2} \quad \text{or} \quad \min_W \max_{\delta \neq 0, \delta \in \mathbb{R}^m} \frac{|| W\delta ||_2}{||\delta||_2}.
\end{equation}
where $\sigma_{\text{max}}(W)$ represents the largest singular value of the matrix $W$. Notably, there is a relation to eigenvalues:
\begin{equation}
    \sigma^2_i(W) = \lambda_i(W W^\dagger) = \lambda_i(W^\dagger W),
\end{equation}
where $W^\dagger$ is the conjugate transpose of $W$. Each singular value of the matrix $W$ is the square root of the eigenvalue of the matrices $WW^\dagger$ or $W^\dagger W$. In other words, minimizing $\lambda_{\text{max}}(W W^\dagger)$, the largest eigenvalue of the matrices, can achieve the same objective.

%-------------------------------------------------------------------------
\subsection{Forged Function}
\label{sec:ours_algo}

One approach to improving the robustness of a neural network is to minimize the objective function defined in (\ref{eq:min_eigen}) to obtain an optimal weight matrix. However, we argue that the largest singular value provides only a loose bound for the Lipschitz constant, as adversarial examples typically lie within an $\epsilon$-ball with a small radius, rather than spanning the entire input space. This observation suggests an alternative strategy: remapping the input domain to a constrained range, which effectively shrinks the reachable space. The restricted set should suppress values in specific dimensions most affected by adversarial examples, thereby minimizing the distribution mismatch between natural and adversarial data, all while maintaining accuracy. Therefore, the problem can be decomposed into two parts: (a) identifying an appropriate constrained range, and (b) transforming it into a corresponding weight matrix that ensures a bounded Lipschitz constant.

We propos a forged function for each layer defined as follows:
\begin{equation}
\label{eq:forge}
  f^{\text{forge}}_i(x) =
    \begin{cases}
      x f^{\text{masked}}_i(x) \quad & \text{if} \quad |x| \leq c^{\text{th}}_i, \\
      x \quad & \text{otherwise},
    \end{cases} 
\end{equation}
where $c^{\text{th}}_i$ is a threshold for the $i$-th layer; $f^{\text{masked}}_i(x)$ is a scaling function with values in the range $[0,1]$. Compared with the original input, the range of the forged function is suppressed if its value is less than the threshold. When $c^{\text{th}}_i$ is set to $0$, the forged function degrades into the original function.

The simplest definition of scaling function is a step function:
\begin{equation}
\label{eq:forge_step}
  f^{\text{step}}(x) =
    \begin{cases}
      0 \quad & |x| \leq c^{\text{th}}_i, \\
      1 \quad & \text{otherwise}.
    \end{cases} 
\end{equation}
However, discontinuous functions might cause unreliable gradients, potentially exposing the defense to BPDA attacks \citep{athalye2018obfuscated}. Therefore, we propose a smooth function that gradually reduces the value to $0$, defined as:
\begin{equation}
\label{eq:forge_smooth}
  f^{\text{smooth}} = \frac{1}{1 + \exp(-ax+b)},
\end{equation}
where $a$ controls the slope, and $b$ determines the location of the transition zone. Alternatively, a piecewise linear function can simplify this behavior:
\begin{equation}
\label{eq:forge_piecewise}
  f^{\text{piecewise}}(x) =
    \begin{cases}
      0 \quad & |x| \leq dc^{\text{th}}_i, \\
      \frac{|x|- dc^{\text{th}}_i}{(1-d)c^{\text{th}}_i}  \quad & dc^{\text{th}}_i \leq |x| \leq c^{\text{th}}_i,\\
      1 \quad & \text{otherwise},
    \end{cases} 
\end{equation}
where $d$ is within the range $[0, 1]$. In (\ref{eq:forge_piecewise}), the transition zone is represented by a linear function, and as a result, the composed $f^{\text{forge}}$ in (\ref{eq:forge}) becomes a quadratic function within this zone.

\begin{figure}
  \centering
  \begin{subfigure}[t]{0.50\linewidth}
    \centering
    \includegraphics[trim=85 50 510 50,clip=true, height=5cm]{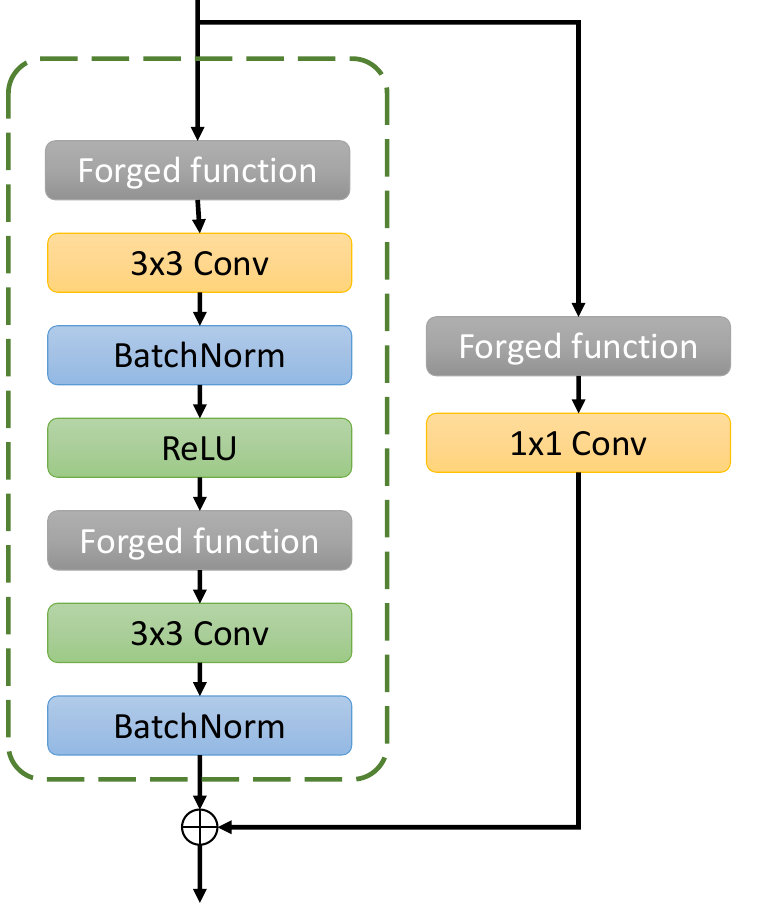}
    \caption{ConvNet}
    \label{fig:cnn}
  \end{subfigure}
  \quad
  \begin{subfigure}[t]{0.40\linewidth}
    \centering
    \includegraphics[trim=650 50 100 50,clip=true, height=5cm]{fig_layer.pdf}
    \caption{Transformer}
    \label{fig:mlp}
  \end{subfigure}
  \caption{Insertion points of the forged function. In ConvNets, it is inserted into the residual blocks, while in Transformers, it is inserted into the MLP layers.}
  \label{fig:forge_in_models}
\end{figure}

Figure \ref{fig:forge_in_models} provides a visual representation of potential insertion points for the forged function, while maintaining the integrity of other layers. For the ResNet architecture, the forged function is placed before the convolutional layers in each residual block. Similarly, for vision transformer architectures, the structure of MLP layers is adapted to seamlessly integrate the forged function. However, although the Query (Q), Key (K), and Value (V) computations in the attention layers involve linear projections, we do not apply the forged function to these components due to the more complex implementation involved.

To understand how the forged function influences the Lipschitz constant, we begin by recalling the Gershgorin Circle Theorem:
\begin{theorem}
\label{theorem:GCT}(Gershgorin Circle Theorem)
For an $m \times m$ matrix $A$ with entries $a_{ij}$, each eigenvalue of $A$ is in at least one of the disk:
\begin{equation}
    R_i = \{ z \in \mathbb{C} : |z - a_{ii}| \leq \sum_{i \neq j}|a_{ij}| \} \quad \mathrm{for} \quad i = \{1, 2, \ldots, m\}.
\end{equation}
\end{theorem}
Theorem \ref{theorem:GCT} indicates each row vector can be represented as a disk which is centered at the diagonal entry $a_{ii}$ and whose radius is the sum of the off-diagonal entries $a_{ij}$. For any layer which can be represented by a linear system, such as convolutional or fully connected layers, robustness can be improved by shrinking the radius of the disk with the largest eigenvalue.

Now, let $W$ be the weight of the target layer, which can be represented by an $m \times n$ matrix, and $\textbf{t}$ be the input vector. Without loss of generality, we assume that $A = W^\dagger W$ and $f^{\text{forge}}(\textbf{t})$ is defined as:
\begin{equation}
f^{\text{forge}}(t_i) =
    \begin{cases}
      0 \quad i \leq h \\
      t_i \quad \text{otherwise},
    \end{cases}
\end{equation}
where $t_i$ is the $i$-th element of the vector $\textbf t$ and $h$ is a positive number.

\begin{lemma}
\label{lemma:ours}
There exists a matrix $A'$ whose largest eigenvalue, $\lambda_{\text{max}}(A')$, is less than or equal to the largest eigenvalue of $A$, $\lambda_{\text{max}}(A)$, if 
\begin{equation}
\label{eq:th_cond}
    Af^{\text{forge}}(\textbf{t})= A'\textbf{t}.
\end{equation}
\end{lemma}

\begin{proof}
Since the first $h$ entries of the vector $\textbf{t}$ are replaced with zeros, above condition can be achieved by replacing the corresponding column vectors of the matrix $A$ with zero vectors. Therefore, the entries of $A'$ are formulated as
\begin{equation}
\label{eq:def_ae_entries}
 a'_{ij} =
    \begin{cases}
      0 \quad j \leq h \\
      a_{ij} \quad \text{otherwise}.
    \end{cases}
\end{equation}
The matrix $A$ is a positive semidefinite matrix, implying that the diagonal entries are non-negative. Moreover, with the entry representation of $A'$ in (\ref{eq:def_ae_entries}), we observe that modifications are only applied to the first $h$ columns, while the rest remain unchanged. Combining the Gershgorin Circle Theorem, we know that the centers of the first $h$ disks of the matrix $A'$ are shifted towards zero. Additionally, the radii of all disks, the absolute values of the off-diagonal entries in $A'$, are shrunk. Consequently, the upper bound of the largest eigenvalue of the matrix $A'$ is tighter compared to that of the original matrix $A$.
\end{proof}

It is important to note that the proof presented above is specified for each individual input vector, with each input having its own corresponding matrix $A'$ that ensures a bounded Lipschitz constant. This input-specific proof allows for precise theoretical guarantees on robustness. While the corresponding matrix $A'$ is auxiliary and used primarily for mathematical proof purposes, there is no need to explicitly compute this matrix during the inference phase. 

\subsection{Data-driven Algorithm}
The choice of a proper $c^{\text{th}}_i$ represents a critical balance in our approach. A larger threshold more effectively constrains the adversarial space and reduces the Lipschitz constant. However, excessive restriction can significantly distort the output distribution, potentially degrading overall accuracy on natural examples.

In this paper, we propose a data-driven approach for determining optimal value through the following equation: 
\begin{equation}
\label{eq:def_cth}
    c^{\text{th}}_i = c^r \max(F_{1 \rightarrow i}(x)) \quad \forall x \in \mathcal{S},
\end{equation}
where $\mathcal{S}$ can include all or a subset of images in the training set, $c^r$ is a positive number within the range $[0,1]$ and $F_{1 \rightarrow i}(x)$ represents the output of the $i$-th layer. Specifically, each layer has its own $c^{\text{th}}_i$, but they share the same hyper-parameter $c^r$. 

Algorithm \ref{algo:forged_algo} specifies the implementation details of the forged function. The variable $b$ is used to store the maximum value that appeared in $\mathcal{S}$, as defined in (\ref{eq:def_cth}), and is initialized during construction. Similar to the implementation of the batchnorm layer, the behavior is depended on the mode configuration. When the mode is set to tracking mode, the variable $b$ is updated accordingly, and the input is set to the output without any modification. Conversely, when the mode is set to inference mode, the value of $b$ is frozen, but the input is updated as defined by (\ref{eq:forge}). By default, the mode is set to inference, and the values of $b$ and $c^r$ are zero, respectively. As a result, the mask $x^{\text{mask}}$ are set to $1$, and the algorithm is degraded to the identical function.

Our method offers significant efficiency advantages in practical implementation as the threshold $c^{\text{th}}_i$ can be obtained by scanning all data in the set $\mathcal{S}$ once in track mode beforehand. This process requires no gradient computations and adds minimal computational overhead—typically just minutes even on consumer-grade GPUs. Compared to the substantial costs of adversarial training, our approach is essentially cost-free while providing meaningful robustness improvements. Conceptually, the proposed forged function operates analogously to a ReLU function, as it suppresses output values within a specific range, but with a key distinction: the defined range in our forged function is adaptively determined based on the characteristics of the observed dataset rather than being fixed. This data-driven adaptation enables our method to enhance robustness while preserving performance on natural examples.

\begin{algorithm}[t]
\caption{Forged Function}
\label{algo:forged_algo}
\begin{algorithmic}[1]
\STATE \textbf{require}: Input $\textbf{x}$,  Mode $m$, Hyper-parameter $c^r$
%\STATE \textbf{note}: $b$ is initialized at stages of construction.
\IF{$m$ is tracking mode}
    \STATE {$b = \max(b, |\textbf{x}|)$}
    %\STATE {$out = x$}
\ELSE
    \FORALL{$x_i \in \textbf{x}$}
        \IF{$\text{abs}(x_i) \leq c^r b$}
            \STATE {$x^{\text{mask}} = f^{\text{masked}}(x)$}
        \ELSE
            \STATE {$x^{\text{mask}}$ = 1}
        \ENDIF
        \STATE {$x_i \leftarrow x^{\text{mask}} x_i$}
    \ENDFOR
    %\STATE {$\mathcal{M} = \left\{ x \middle\vert \text{abs}(x) \leq c^r b \right\}$}
    %\FORALL{$s \in \mathcal{M}$}
    %    \STATE{$s \leftarrow sf^{\text{masked}}(s)$}
    %\ENDFOR
\ENDIF

\STATE {\textbf{return} $\textbf{x}$}
\end{algorithmic}
\end{algorithm}

%%%%% %%%%% %%%%% %%%%% %%%%% %%%%% %%%%% %%%%% %%%%% %%%%% %%%%% %%%%%
\section{Experiments}
\label{sec:exp}
%-------------------------------------------------------------------------
\subsection{Setup}
We evaluated our method on CIFAR10, CIFAR100, and ImageNet datasets under white-box attack scenarios using AutoAttack \citep{croce2020reliable} with $L_\infty$ norm constraints ($\epsilon=8/255$ for CIFAR10/100 and $\epsilon=4/255$ for ImageNet). All baseline models were sourced from the public RobustBench repository for reproducibility.

To determine threshold values $c^{\text{th}}_i$, we processed the complete training sets for CIFAR10 and CIFAR100 without data augmentation. For the larger ImageNet dataset, scanning all 1.2 million training images would require substantial time. Alternatively, we randomly sampled approximately 5,000 representative images to establish thresholds efficiently. Ideally, determining the optimal choice of $c^r$ requires conducting an ablation study to explore the relationship between the chosen $c^r$ and robust accuracy on a validation set. To accelerate this procedure, we first seek a value of $c^r$ with the highest standard accuracy. The candidate values are selected in a small range centered around this value.

Since most models are represented in 16 bit format, and the widths of fraction bit for FP16 format defined by IEEE-754 standard and BFloat are 10 and 7 bits, respectively, truncated errors might easily occur when performing addition on two numbers with a magnitude difference of $2^8$ or higher. On the other hand, when $c^r$ is set to $2^{-5}$, all models experience a significant drop in standard accuracy, and there is meaningless in evaluating robustness at this configuration. We suggest that the candidates of $c^r$ are $2^{-8}$, $2^{-7}$ and $2^{-6}$ and the partial results are presented in Table \ref{table:partial_select_cr_CIFAR10}.

Our experimental investigation included ablation studies examining: optimal scaling coefficient selection, compatibility with various adversarially trained models, verification against gradient masking, and certified robustness through randomized smoothing. Complete ablation study results are presented in Appendix \ref{sec:appd:exp} due to space constraints.

%-------------------------------------------------------------------------
\subsection{White-box Evaluation}
\begin{table}[t]
    \centering
    \begin{subtable}[t]{.99\linewidth}

        \centering
        %\raggedright
        %\begin{adjustbox}{width=1.0\textwidth}
        \begin{tabular}{ccc|cc}
            \toprule
            \multirow{2}{*}{Method} & \multicolumn{2}{c|}{Original} & \multicolumn{2}{c}{Original+Ours} \\
                                    & acc\textsubscript {nat} & acc\textsubscript {AA}      & acc\textsubscript {nat} & acc\textsubscript {AA}      \\
              \midrule
            RST-AWP \cite{wu2020adversarial} & 88.25 & 60.04 & 89.50 & 62.76 \\
            DefEAT  \cite{chen2024data}      & 86.54 & 57.30 & 87.40 & 59.55 \\
            LTD     \cite{chen2021ltd}       & 85.21 & 56.94 & 85.98 & 59.25 \\
            AWP     \cite{wu2020adversarial} & 85.36 & 56.17 & 86.19 & 57.85 \\
              \midrule
            TRADES\cite{zhang2019theoretically} & 85.34 & 52.86 & 85.78 & 53.80 \\
            \bottomrule
        \end{tabular}
        \caption{CIFAR10 dataset}
        \label{table:acc_ours_CIFAR10}
        %\end{adjustbox}
    \end{subtable}%
    \newline
    \bigskip
    \newline
    %\hfill
    \begin{subtable}[t]{.99\linewidth}
        \centering
        %\raggedleft
        
        %\begin{adjustbox}{width=1.0\textwidth}
        \begin{tabular}{ccc|cc}
            \toprule
            \multirow{2}{*}{Method} & \multicolumn{2}{c|}{Original} & \multicolumn{2}{c}{Original+Ours} \\
                                    & acc\textsubscript {nat} & acc\textsubscript {AA}      & acc\textsubscript {nat} & acc\textsubscript {AA}      \\
              \midrule
            EffAug  \cite{addepalli2022efficient} & 68.75 & 31.85 & 69.14 & 32.57 \\ 
            DKLD    \cite{cui2023decoupled}       & 64.08 & 31.65 & 64.26 & 32.58 \\
            DefEAT  \cite{chen2024data}           & 64.32 & 31.13 & 66.42 & 32.11 \\
            LTD     \cite{chen2021ltd}            & 64.07 & 30.59 & 64.29 & 31.95 \\
            AWP     \cite{wu2020adversarial}      & 60.38 & 28.86 & 60.63 & 29.72 \\
            \bottomrule
        \end{tabular}
        \caption{CIFAR100 dataset}
        \label{table:acc_ours_CIFAR100}
        %\end{adjustbox}
    \end{subtable}
    \caption{Standard and robust accuracy of models integrated with our method.}
\end{table}

This experiment evaluates the generalizability of the proposed method on adversarially trained models with the same architecture but trained using different strategies. Additionally, we assess the potential cost reduction from adversarial training. We integrated our approach with models selected from RobustBench, trained using various techniques such as adding perturbations in internal layers, retrieving information using knowledge distillation, reducing inefficient training data, or involving additional images from generated models or another dataset. Except for the WRN-28-10 model used in RST-WAP, all models were based on the popular WRN-34-10 with ReLU, the leading architecture on the RobustBench leaderboard \citep{croce2020robustbench}. The value of $c^r$ is set to $2^{-7}$.

Table \ref{table:acc_ours_CIFAR10} and \ref{table:acc_ours_CIFAR100} present standard and robust accuracy of models integrated with our method for CIFAR10 and CIFAR100 dataset, respectively. In these tables, \textit{acc\textsubscript{nat}} and \textit{acc\textsubscript{AA}} refer to accuracy on clean data and adversarial examples generated by AutoAttack, respectively. The column \textit{Original} indicates the original results reported by RobustBench, and the column \textit{Original+Ours} demonstrates the results of the proposed method. As indicated in these tables, for CIFAR10 dataset, the proposed method enhances robust accuracy by more than $2\%$ for RST-AWP, DefEAT, and LTD models, while other models receive approximately $1$ to $1.5\%$ improvement in robustness. Similarly, for CIFAR100 dataset, these models meet at least a $1\%$ increase in robustness. The empirical results prove that the resilience of existing models against adversarial attacks can be improved by Lemma \ref{lemma:ours}. We believe that the proposed solution is general as it achieves great success in models incorporating different training techniques.

Additionally, models integrated with our approach incur significantly lower training costs compared to models with similar robustness. For example, combining the LefEAT model with our method achieves a robust accuracy of 59.55\% can be achieved, which is comparable to RST-AWP (60.04\%). However, RST-AWP introduces more images from another dataset, resulting in a higher cost in each epoch. Similarly, for the CIFAR100 dataset, DefEAT with our proposed method achieves a robust accuracy of 32.11\%, which is better than EffAug (31.85\%), which involves more complex data augmentation during the training stage. This result supports DefEAT's suggestion that unnecessary data can be eliminated without sacrificing robustness. Overall, we believe that our approach offers a valuable strategy for pruning inefficient weights in the later stages of adversarial training, potentially leading to further cost savings or enhanced resilience.

An interesting observation from the tables is that standard accuracy improves for all models on both CIFAR10 and CIFAR100. While this is not directly explained by Lemma \ref{lemma:ours}, we hypothesize that the output ranges of ReLU and our proposed function are similar, helping maintain accuracy on clean data.

Tables \ref{table:top_3_robustbench_cifar10}, \ref{table:top_3_robustbench_cifar100}, and \ref{table:top_3_robustbench_imagenet} present the top-3 models on RobustBench for the CIFAR10, CIFAR100, and ImageNet datasets, respectively. The rankings are denoted by $\#$, with our results marked by an asterisk (*).

For CIFAR10 and CIFAR100 datasets, we used the WRN-70-16 architecture with SiLU activations, incorporating generative data during training. As shown, our approach, when applied to the WRN-70-16 model with SiLU, improves robustness by at least 0.9\%, achieving the best results on RobustBench for both datasets. However, there is a slight decrease in standard accuracy (\textit{acc\textsubscript{nat}}). Several factors may explain this outcome. For instance, the additional hyperparameter introduced in this study might not offer sufficient granularity to optimize all layers of the target model effectively. Although the proposed method was originally designed for ReLU activations, these results indicate its applicability to other activation functions. Future work could explore the design of more specialized functions to further optimize performance with different activation function.

On the other hand, For the ImageNet dataset, we employed the Swin model \citep{liu2021swin}, a transformer-based architecture. Experimental results show that the Swin-L model with GELU activations, combined with our method, improves robust accuracy and achieves the best result, while standard accuracy sees only a marginal decrease. This suggests that our approach is effective across both convolutional and transformer-based architectures.

\begin{table}[t]
  \centering
  \begin{subtable}[t]{.99\linewidth}
      \centering
      %\begin{adjustbox}{width=1.0\linewidth}
      \begin{tabular}{c|cccc}
          \toprule
          \# & Method & Architecture & acc\textsubscript {nat} & acc\textsubscript {AA} \\
          \midrule
          * & \cite{wang2023better} + Ours & WRN-70-16 & 93.20 & \textbf{71.70} \\ 
          1 & \cite{peng2023robust}        & WRN-70-16 & 93.27 & 71.07 \\
          2 & \cite{wang2023better}        & WRN-70-16 & 93.25 & 70.69 \\
          3 & \cite{bai2024mixednuts}      & WRN-70-16 & \textbf{95.19} & 69.71 \\
        \bottomrule
      \end{tabular}
      %\end{adjustbox}
      \caption{CIFAR10 dataset}
      \label{table:top_3_robustbench_cifar10}
  \end{subtable}%
    \newline
    \bigskip
    \newline
  \begin{subtable}[t]{.99\linewidth}
      \centering

      %\begin{adjustbox}{width=1.0\linewidth}
      \begin{tabular}{c|cccc}
        \toprule
        \# & Method & Architecture & acc\textsubscript {nat} & acc\textsubscript {AA} \\
        \midrule
          * & \cite{wang2023better} + Ours & WRN-70-16 & 74.97 & \textbf{44.00} \\ 
          1 & \cite{wang2023better}        & WRN-70-16 & \textbf{75.22} & 42.67 \\
          2 & \cite{bai2024mixednuts}      & WRN-70-16 & 83.08 & 41.80 \\
          3 & \cite{cui2023decoupled}      & WRN-70-16 & 73.85 & 39.18 \\
      \bottomrule
    \end{tabular}
    %\end{adjustbox}
    \caption{CIFAR100 dataset}
    \label{table:top_3_robustbench_cifar100}
  \end{subtable}
  \begin{subtable}[t]{.99\linewidth}
      \centering
      \begin{adjustbox}{width=1.0\linewidth}
      \begin{tabular}{c|cccc}
          \toprule
          \# & Method & Architecutre & acc\textsubscript {nat} & acc\textsubscript {AA} \\
          \midrule
          * & \cite{liu2023comprehensive} + Ours & Swin-L                & 78.88 & \textbf{60.04} \\ 
          1 & \cite{liu2023comprehensive}        & Swin-L                & 78.92 & 59.56 \\
          2 & \cite{bai2024mixednuts}            & ConvNeXtV2-L + Swin-L & \textbf{81.48} & 58.50 \\
          3 & \cite{liu2023comprehensive}        & ConvNeXt-L            & 78.02 & 58.48 \\
        \bottomrule
      \end{tabular}
      \end{adjustbox}
  \caption{ImageNet dataset}
  \label{table:top_3_robustbench_imagenet}
  \end{subtable}
  \caption{The results of top-3 competitors on Robustbench.}
\end{table}

\begin{table*}[t]
\centering
\begin{tabular}{c|cc|cc|cc|cc}
\toprule
\multirow{2}{*}{Method} & \multicolumn{2}{|c|}{RobustBench} & \multicolumn{2}{c|}{$c^r=2^{-8}$} & \multicolumn{2}{c|}{$c^r=2^{-7}$} & \multicolumn{2}{c}{$c^r=2^{-6}$} \\
                        & acc\textsubscript{nat} & acc\textsubscript{AA} & acc\textsubscript{nat} & acc\textsubscript{AA} & acc\textsubscript{nat} & acc\textsubscript{AA} & acc\textsubscript{nat} & acc\textsubscript{AA}  \\
\midrule
TRADES \cite{zhang2019theoretically}                 & 85.34 & 52.86 & 85.57 & 52.97 & 85.78 & 53.80 & 85.49 & 55.37 \\
\bottomrule
\end{tabular}
  \caption{Ablation study of selecting optimal $c^r$ for CIFAR10 dataset.}
  \label{table:partial_select_cr_CIFAR10}
\end{table*}

%-------------------------------------------------------------------------
\subsection{Cost Analysis}
For the CIFAR-10 and CIFAR-100 datasets, the total training time for AWP with WRN-34-10 on a V100 GPU is approximately 20 hours. In contrast, other methods use larger models and incorporate additional data \citep{wang2023better, peng2023robust, bai2024mixednuts}, leading to a total training time of over 200 hours. In comparison, scanning the entire CIFAR-10 or CIFAR-100 dataset, or partial training set on ImageNet dataset, takes less than 5 minutes on a V100 GPU.

Regarding the costs of hyperparameter search, as discussed in Appendix \ref{sec:hyper-para_select}, there are only three candidate parameters to evaluate performance. We can efficiently assess white-box performance using partial data from the training set, which significantly reduces the computational time. In practice, the total time for hyperparameter search is about 1-3 hours, depending on the model size and the dataset used. This demonstrates that our approach incurs significantly lower computational overhead.

%-------------------------------------------------------------------------

\subsubsection{Gradient Masking Verification}
Previous studies suggest that the resilience of models might be unintentionally overestimated \citep{athalye2018obfuscated,carlini2019evaluating}. The proposed function in (\ref{eq:forge}) suppresses values to zero if the condition is satisfied. One might argue that this property could unintentionally cause obfuscated gradients, resulting in gradient attacks being unable to efficiently produce adversarial examples. Therefore, to verify that the proposed method does not encounter the gradient masking issue, we should conduct more experiments from the following aspects:
\begin{enumerate}
    \item\label{item:A} White-box attacks should be better than black-box attacks.
    \item\label{item:B} Iterative attacks should have better performance than one-step attacks.
    \item\label{item:C} Robust accuracy should gradually decrease to zero when the radius of $\epsilon$-ball increase.
    \item\label{item:D} The modified model should defense against adversarial examples generated by the original models.
    \item\label{item:E} Certified robustness that conducted by random smoothing \citep{cohen2019certified}.
\end{enumerate}

AutoAttack has examined the first item, which involves three white-box attacks and one black-box attack. By comparing robust accuracy shown in Table \ref{table:top_3_robustbench_cifar10}, \ref{table:top_3_robustbench_cifar100} and \ref{table:top_3_robustbench_imagenet}, the models combined with the proposed method perform better robust accuracy than the original models. It indicates black-box attacks cannot produce more adversarial examples. 

The full experimental results for the rest of the experiments can be found in Appendix \ref{sec:appd:masking}. The results demonstrate that the proposed algorithm does not violate any of the above rules and the certified robustness improves by our method across most settings. From the evidence, we believe that the proposed method does not encounter the gradient masking problem among different hyper-parameters and various models on CIFAR10 and CIFAR100 datasets. % \ref{sec:appd:masking}

\subsection{Discussion}
\textbf{Comparison with Lipschitz-constrained methods} A common misconception is that minimizing the Lipschitz constant necessarily improves model robustness. However, in extreme cases — for example, replacing the weights of a linear layer with an identity matrix — the Lipschitz constant becomes 1, but the resulting model may fail to generalize, yielding zero natural accuracy and thus no meaningful robustness to assess. While certified training \citep{mao2024understanding} and Lipschitz-constrained methods \citep{zuhlke2024adversarial} can reduce the Lipschitz constant, they often introduce additional regularization terms or complex objectives, making it difficult to balance robustness and performance. In contrast, our approach enhances robustness by shrinking the space of adversarial examples, without requiring explicit minimization as in (\ref{eq:min_eigen}).

\textbf{Effective adversarial training} The proposed function can be constructed in a single pass over the training data, without generating adversarial examples. By suppressing values in regions associated with adversarial perturbations, our method reduces the distributional mismatch between clean and adversarial data. Although similar ideas have been explored in prior work \citep{bai2021improving}, such methods typically rely on explicit adversarial training. Additionally, the empirical results show that when applied to models trained without additional generative data, our method achieves robustness comparable to or better than models that use extensive supplementary data. This implicitly suggests that unnecessary data can be eliminated without without degrading prediction quality which is aligned with previous work \cite{chen2024data}.

\textbf{Structured pruning} The overall procedure of the proposed algorithm shares many similarities with pruning techniques \citep{he2023structured}. However, we would like to emphasize that the proposals are distinctly different. Pruning primarily aims to create a highly sparse model that accelerates inference times or reduces the model size for deployment on edge devices, often without considering robustness. In particular, when the pruned model exhibits extremely high sparsity without applying re-training or fine-tuning, natural accuracy drops significantly, implicitly indicating that condition (\ref{eq:th_cond}) does not hold and that Lemma \ref{lemma:ours} is not applicable in this case. On the other hand, when the pruned model has low sparsity, only a small proportion of weights with values close to zero are eliminated, resulting in insignificant adjustments to the center and radius of the disk. The proposed algorithm, on the other hand, can identify crucial elements with minimal cost. We believe that investigating the impact of various pruning techniques, such as iterative or post-training methods, on robustness, or combining these techniques with our proposed approach, represents a valuable direction for future research.

%%%%% %%%%% %%%%% %%%%% %%%%% %%%%% %%%%% %%%%% %%%%% %%%%% %%%%% %%%%%
\section{Conclusion}
In this paper, we revisited robustness certification through the concept of Lipschitz continuity. While existing Lipschitz-constrained methods reduce the Lipschitz constant using regularization or complex objectives, they often struggle to strike a balance between robustness and performance. In contrast, from a theoretical perspective, we demonstrated that remapping inputs to a constrained set can mitigate the influence of regions outside the real data distribution, leading to improved robustness.

Our method is data-driven and highly efficient, requiring only a single pass through the dataset to determine the appropriate parameters for constructing the forged function without fine-tuning or gradient computations. This procedure is computationally efficient and almost cost-free. The experimental results highlight that our method can be seamlessly integrated with various existing techniques to enhance robustness. When combined with models trained without additional generative data, our method achieves robustness comparable to or exceeding that of models using extensive supplementary data. These results open a promising direction for significantly reducing computational costs while maintaining or improving defensive capabilities of robust neural networks.

Numerous future directions merit exploration. For instance, investigating the integration of the proposed remapping function with different activation functions, model architectures, or large-scale datasets would be beneficial. Additionally, it would be worthwhile to explore the theoretical foundations behind the observed improvements in standard accuracy introduced by our method.

\section*{Acknowledgments}
{
We thank the founding support from National Science and Technology Council  ( 113-2622-E-007 -018 ) and the computing infrastructure provided by the National Center for High-Performance Computing, National Institutes of Applied Research (NIAR), Taiwan.

    \small
    \bibliographystyle{ieeenat_fullname}
    \bibliography{refs}
}

\appendix
%\iffalse
\section{Ablation Study}
\label{sec:appd:exp}

\subsection{Hyper-parameter Selection}
\label{sec:hyper-para_select}
This experiment investigates how the choice of hyper-parameter $c^r$ influences standard accuracy and robust accuracy. Since most models are represented in 16 bit format, and the widths of fraction bit for FP16 format defined by IEEE-754 standard and BFloat are 10 and 7 bits, respectively, truncated errors might easily occur when performing addition on two numbers with a magnitude difference of $2^8$ or higher. On the other hand, when $c^r$ is set to $2^{-5}$, all models experience a significant drop in standard accuracy, and there is meaningless in evaluating robustness at this configuration. We suggest that the candidates of $c^r$ are $2^{-8}$, $2^{-7}$ and $2^{-6}$. 

The results on CIFAR10 and CIFAR100 are presented in Table \ref{table:select_cr_CIFAR10} and Table \ref{table:select_cr_CIFAR100}, respectively. Moreover, the results of accuracy against CW attack on $L_\infty$ norm for CIFAR10 and CIFAR100 datasets are presented in Tables \ref{table:CWinf_CIFAR10} and \ref{table:CWinf_CIFAR100}, respectively. As can be seen, when $c^r$ is set to $2^{-8}$, all models achieve better standard accuracy and robust accuracy. Additionally, the results for all models with $c^r=2^{-7}$ are surpassed by those when $c^r$ is set to $2^{-8}$. Robust accuracy can be further enhanced by setting $2^{-6}$, while standard accuracy might drop compared to the original. The results suggest that $c^r=2^{-7}$ is a solution that balances standard accuracy and robustness. Nevertheless, when robustness is a major concern, $c^r=2^{-6}$ is a better choice. 

Intuitively, we expect that standard accuracy gradually decreases when the value of $c^r$ increases. The phenomenon can be observed when $c^r$ is $2^{-6}$ or higher but two counterexamples are reported in the ablation study when setting $c^r$ to $2^{-7}$ and $2^{-8}$. A possible explanation is that the optimizer becomes stuck in a saddle area, as ReLU is non-differentiable at the zero point. This might cause the gradient direction to become stuck in an oscillation when values are close to zero. By shifting those values to zero, antagonistic effects among different feature maps, filters, or channels are accidentally mitigated. However, further investigation and evidence are needed to support this conjecture.

We argue that any function that satisfies the conditions defined in (13) can shrink the largest eigenvalue. There might be another function that can perform better than the proposed one. Besides, the hyper-parameter is determined by choosing the maximum value appearing in the dataset.

\begin{table*}[t]
\centering
\begin{tabular}{c|cc|cc|cc|cc}
\toprule
\multirow{2}{*}{Method} & \multicolumn{2}{|c|}{RobustBench} & \multicolumn{2}{c|}{$c^r=2^{-8}$} & \multicolumn{2}{c|}{$c^r=2^{-7}$} & \multicolumn{2}{c}{$c^r=2^{-6}$} \\
                        & acc\textsubscript{nat} & acc\textsubscript{AA} & acc\textsubscript{nat} & acc\textsubscript{AA} & acc\textsubscript{nat} & acc\textsubscript{AA} & acc\textsubscript{nat} & acc\textsubscript{AA}  \\
\midrule
RST-AWP                 & 88.25 & 60.04 & 88.82 & 60.96 & 89.50 & 62.76 & 87.88 & 61.96 \\
DefEAT                  & 86.54 & 57.30 & 86.88 & 57.81 & 87.40 & 59.55 & 84.59 & 61.08 \\
LTD                     & 85.21 & 56.94 & 85.28 & 57.28 & 85.98 & 59.25 & 85.59 & 60.63 \\
AWP                     & 85.36 & 56.17 & 85.80 & 56.53 & 86.19 & 57.85 & 84.55 & 59.21\\
\midrule
TRADES                  & 85.34 & 52.86 & 85.57 & 52.97 & 85.78 & 53.80 & 85.49 & 55.37 \\
\bottomrule
\end{tabular}
  \caption{Ablation study on selecting optimal $c^r$ for CIFAR10 dataset.}
  \label{table:select_cr_CIFAR10}
\end{table*}

\begin{table*}[t]
\centering
\begin{tabular}{c|cc|cc|cc|cc}
\toprule
\multirow{2}{*}{Method} & \multicolumn{2}{|c|}{RobustBench} & \multicolumn{2}{c|}{$c^r=2^{-8}$} & \multicolumn{2}{c|}{$c^r=2^{-7}$} & \multicolumn{2}{c}{$c^r=2^{-6}$} \\
                        & acc\textsubscript{nat} & acc\textsubscript{AA} & acc\textsubscript{nat} & acc\textsubscript{AA} & acc\textsubscript{nat} & acc\textsubscript{AA} & acc\textsubscript{nat} & acc\textsubscript{AA}  \\
\midrule
EffAug                  & 68.75 & 31.85 & 68.81 & 32.00 & 69.14 & 32.57 & 68.44 & 33.64 \\ 
DKLD                    & 64.08 & 31.65 & 64.10 & 31.77 & 64.26 & 32.58 & 63.50 & 33.87 \\
DefEAT                  & 65.89 & 30.57 & 66.12 & 31.11 & 66.42 & 32.46 & 65.06 & 34.07 \\
LTD                     & 64.07 & 30.59 & 64.29 & 31.13 & 64.29 & 31.95 & 64.18 & 34.04 \\
AWP                     & 60.38 & 28.86 & 60.18 & 29.10 & 60.63 & 29.72 & 60.71 & 30.82 \\
\bottomrule
\end{tabular}
  \caption{Ablation study on selecting optimal $c^r$ for CIFAR100 dataset.}
  \label{table:select_cr_CIFAR100}
\end{table*}

\subsection{Scaling Training Data for Hyperparameter Tuning}
This experiment investigates how the number of training images used for hyperparameter selection affects both standard accuracy and robust accuracy. As shown in Table~\ref{table:top_3_robustbench_imagenet_ablation}, robust accuracy decreases slightly as the amount of training data increases, but it consistently remains higher than that of the competing methods.

\begin{table}[ht]
  \centering

    \centering
    \begin{adjustbox}{width=1.0\linewidth}
      \begin{tabular}{c|cccc}
        \toprule
        \# & Method & Training Data [\#] & acc\textsubscript{nat} & acc\textsubscript{AA} \\
        \midrule
        1 & \cite{liu2023comprehensive} + Ours & 5,000  & 78.88 & \textbf{60.04} \\ 
        2 & \cite{liu2023comprehensive} + Ours & 10,000 & 78.92 & 59.97 \\ 
        3 & \cite{liu2023comprehensive} + Ours & 20,000 & 78.90 & 59.81 \\ 
        \hline
        4 & \cite{liu2023comprehensive}        & -                 & 78.92 & 59.56 \\
        5 & \cite{bai2024mixednuts}            & - & \textbf{81.48} & 58.50 \\
        \bottomrule
      \end{tabular}
    \end{adjustbox}
  \caption{Ablation study on the impact of scaling training data for ImageNet dataset.}
  \label{table:top_3_robustbench_imagenet_ablation}
\end{table}

%-------------
\section{Full Experimental Results of Gradient Masking Verification}
\label{sec:appd:masking}
\begin{table}[t]
  \centering
  \begin{subtable}[t]{.99\linewidth}
      \centering
      \begin{tabular}{c|c|ccc}
        \toprule
        \multirow{2}{*}{Method} & \multirow{2}{*}{Origin} & \multicolumn{3}{c}{$c^r$} \\
                                &                         & $2^{-8}$ & $2^{-7}$ & $2^{-6}$  \\
        \midrule
        RST-AWP                 & 58.98 & 61.84 & 68.24 & 80.92 \\
        DefEAT                  & 56.92 & 58.02 & 61.06 & 65.56 \\
        LTD                     & 58.12 & 58.56 & 60.50 & 64.86 \\
        AWP                     & 56.84 & 57.34 & 60.58 & 66.50 \\
        \midrule
        TRADES                  & 56.10 & 56.52 & 58.18 & 63.62 \\
        \bottomrule
      \end{tabular}
      \caption{CIFAR10 dataset}
      \label{table:CWinf_CIFAR10}
  \end{subtable}%
  \newline
  \bigskip
  \newline
  \begin{subtable}[t]{.99\linewidth}
      \centering
      \begin{tabular}{c|c|ccc}
        \toprule
        \multirow{2}{*}{Method} & \multirow{2}{*}{Origin} & \multicolumn{3}{c}{$c^r$} \\
                                &                         & $2^{-8}$  & $2^{-7}$  & $2^{-6}$  \\
        \midrule
        EffAug                  & 37.40 & 37.70 & 38.70 & 43.00 \\ 
        DKLD                    & 37.50 & 38.06 & 39.38 & 44.20 \\
        DefEAT                  & 36.90 & 37.56 & 39.82 & 44.30 \\
        LTD                     & 36.66 & 37.32 & 38.86 & 43.44 \\
        AWP                     & 34.56 & 35.20 & 35.94 & 40.40 \\
        \bottomrule
      \end{tabular}
      \caption{CIFAR100 dataset}
      \label{table:CWinf_CIFAR100}
  \end{subtable}
  \caption{The robust accuracy against CW attack on $L_\infty$ norm.}
\end{table}

Table \ref{table:transfer_CIFAR10} and \ref{table:transfer_CIFAR100} present the robust accuracy against adversarial examples generated by the original models on CIFAR10 and CIFAR100 datasets, respectively. As observed, none of the models showed lower robust accuracy than the original model. It indicates that adversarial examples can be efficiently crafted by utilizing the gradients from the victim models.

Table \ref{table:full_acc_FGSM_PGD_CIFAR10} and \ref{table:full_acc_FGSM_PGD_CIFAR100} presents the robust accuracy against FGSM and PGD attacks among different radii of the $\epsilon$-ball on the CIFAR10 and CIFAR100 datasets, respectively. As observed, the robust accuracy against FGSM, a one-step attack, is always higher than the robust accuracy against PGD, an iterative attack. This implies that the gradient is reliable, allowing the PGD attack to adjust the gradient direction multiple times to find adversarial examples. Additionally, we observe that the robust accuracy against PGD attacks for all models gradually decreases to zero as the radius of the $\epsilon$-ball increases. This indicates that the quality of gradients is preserved, enabling PGD attacks to move the gradient toward examples not in the observed distribution.

Figure \ref{fig:rs_cifar10} illustrates the certified robustness achieved by random smoothing for various models on the CIFAR10 dataset, where \textit{Original} refers to the certified robustness of the original model, while \textit{Ours} denotes the robustness of the model combined with the proposed method. As can be seen, our method brings slight improvements in robustness, except for the AWP model. These results demonstrate that our algorithm does not suffer from the gradient masking issue. However, the empirical Lipschitz constant is derived from the observed data. As the input distribution drawn from random smoothing and the observed data might have discrepancies, this could result in fluctuations in robustness.

Alternatively, we can empirically assess the gradient masking using Lemma 2. The first step is the same as usual: obtaining the hyperparameter $c^{\text{th}_i}$ for each layer. The second step involves counting the occurrences of specific elements in the input vectors that are remapped to $0$. The corresponding columns of the weight matrix are then replaced with zero vectors. Table \ref{table:robustbench_imagenet_BPDA} presents the results, where \textit{\cite{liu2023comprehensive} + Ours + Weight} refers to the alternative implementation. As shown, the alternative implementation achieves better robust accuracy compared to the original model. However, the proposed method outperforms the alternative implementation. This is because the proposed method offers the flexibility to remap the input vectors on a sample-wise granularity, while the alternative implementation replaces the corresponding columns of the weight matrix, which has a broader influence on all samples. % \ref{lemma:ours}

\begin{table}[t]
      \centering
      \begin{adjustbox}{width=1.0\linewidth}
      \begin{tabular}{cccc}
          \toprule
          Method & Architecutre & acc\textsubscript {nat} & acc\textsubscript {AA} \\
          \midrule
          \cite{liu2023comprehensive} + Ours & Swin-L                & 78.88 & 60.04 \\ 
          \cite{liu2023comprehensive} + Ours + Weight & Swin-L       & 78.82 & 59.71 \\ 
          \cite{liu2023comprehensive}        & Swin-L                & 78.92 & 59.56 \\
        \bottomrule
      \end{tabular}
      \end{adjustbox}
  \caption{Gradient masking verification by the weight adjustment.}
  \label{table:robustbench_imagenet_BPDA}
\end{table}

\begin{figure*}[t]
    \centering
    \begin{subfigure}{.48\textwidth}
        \centering
        \includegraphics[width=.95\linewidth]{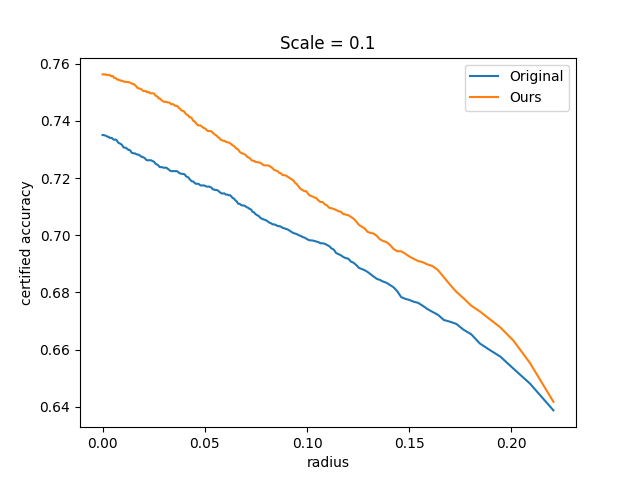}
        \caption{RST-AWP}
    \end{subfigure}
    \begin{subfigure}{0.48\textwidth}
        \centering
        \includegraphics[width=.95\linewidth]{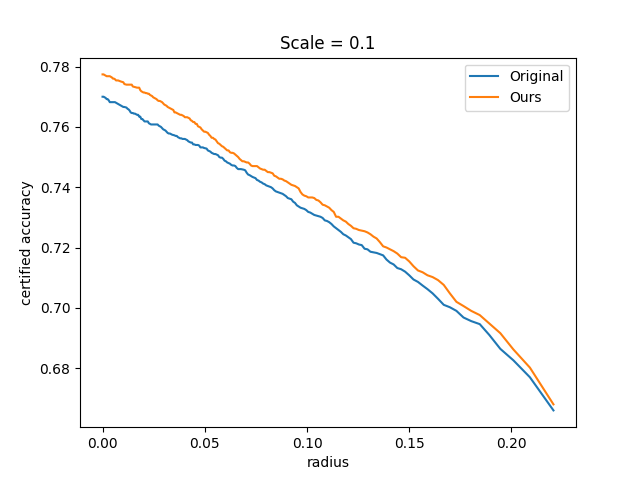}
        \caption{DefEAT}
    \end{subfigure}
    \begin{subfigure}{.48\textwidth}
        \centering
        \includegraphics[width=.95\linewidth]{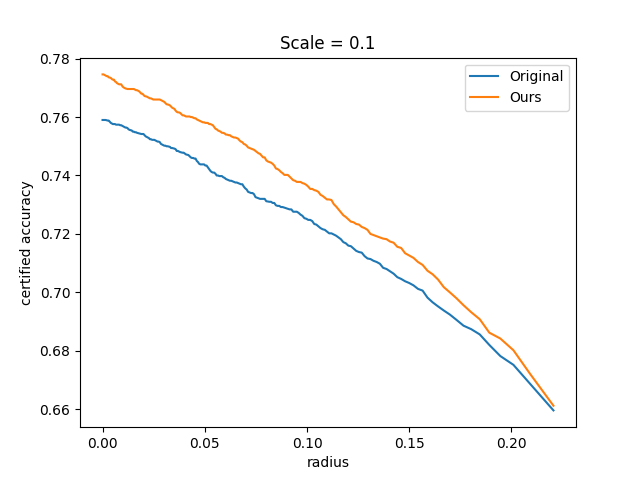}
        \caption{LTD}
    \end{subfigure}
    \begin{subfigure}{0.48\textwidth}
        \centering
        \includegraphics[width=.95\linewidth]{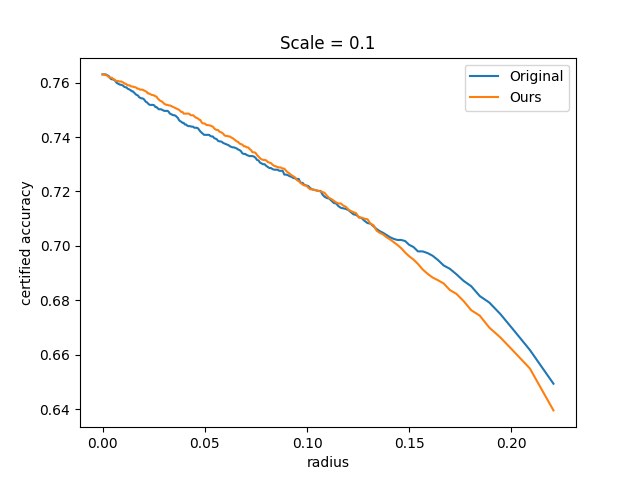}
        \caption{AWP}
    \end{subfigure}
    \caption{Certified robustness that conducted by random smoothing.}
    \label{fig:rs_cifar10}
\end{figure*}

\begin{table}[t]
  \begin{subtable}[t]{.99\linewidth}
      \centering
      \begin{tabular}{c|c|ccc}
        \toprule
        \multirow{2}{*}{Method} & \multirow{2}{*}{Origin} & \multicolumn{3}{c}{$c^r$} \\
                                &                         & $2^{-8}$ & $2^{-7}$ & $2^{-6}$  \\
        \midrule
        RST-AWP                 & 60.04 & 62.10 & 65.10 & 70.53 \\
        DefEAT                  & 57.30 & 58.37 & 60.39 & 66.10 \\
        LTD                     & 56.94 & 58.71 & 61.63 & 66.47 \\
        AWP                     & 56.17 & 57.49 & 59.74 & 65.58 \\
        \midrule
        TRADES                  & 52.86 & 55.55 & 55.09 & 58.68 \\
        \bottomrule
      \end{tabular}
      \caption{CIFAR10 dataset}
      \label{table:transfer_CIFAR10}
  \end{subtable}%
  \newline
  \bigskip
  \newline
  \begin{subtable}[t]{.99\linewidth}
      \centering
      \begin{tabular}{c|c|ccc}
        \toprule
        \multirow{2}{*}{Method} & \multirow{2}{*}{Origin} & \multicolumn{3}{c}{$c^r$} \\
                                &                         & $2^{-8}$  & $2^{-7}$  & $2^{-6}$  \\
        \midrule
        EffAug                  & 31.85 & 32.87 & 35.08 & 40.04 \\ 
        DKLD                    & 31.65 & 32.91 & 35.04 & 40.58 \\
        DefEAT                  & 30.57 & 31.82 & 33.94 & 40.67 \\
        LTD                     & 30.59 & 32.05 & 34.07 & 39.11 \\
        AWP                     & 28.86 & 29.88 & 32.18 & 36.67 \\
        \bottomrule
      \end{tabular}
      \caption{CIFAR100 dataset}
      \label{table:transfer_CIFAR100}
  \end{subtable}
  \caption{The robust accuracy against adversarial examples generated by the original models.}
\end{table}

\begin{table*}[t]
\centering
\begin{tabular}{ccc|cccccccccccccc}
\toprule
\multirow{2}{*}{Method} & \multirow{2}{*}{$c^r$} & \multirow{2}{*}{Attack} & \multicolumn{8}{c}{$\epsilon$} \\
                        &                        &                         & $\frac{1}{255}$ & $\frac{2}{255}$ & $\frac{4}{255}$ & $\frac{8}{255}$ & $\frac{16}{255}$ & $\frac{32}{255}$ & $\frac{64}{255}$ & $\frac{96}{255}$ \\
\midrule
\multirow{6}{*}{RST-AWP}& \multirow{2}{*}{$2^{-8}$} & FGSM & 88.28 & 86.94 & 83.80 & 75.12 & 57.23 & 34.04 & 18.80 & 19.03 \\
                        &                           & PGD  & 86.78 & 84.47 & 79.03 & 66.03 & 34.02 & 2.01  & 0.01  & 0.0   \\
                        \cmidrule(lr){2-11}
                        & \multirow{2}{*}{$2^{-7}$} & FGSM & 89.46 & 88.28 & 85.64 & 77.62 & 60.60 & 35.91 & 19.39 & 20.68 \\
                        &                           & PGD  & 88.03 & 85.88 & 80.89 & 69.24 & 38.27 & 3.08  & 0.1   & 0.0   \\
                        \cmidrule(lr){2-11}
                        & \multirow{2}{*}{$2^{-6}$} & FGSM & 87.70 & 86.63 & 84.38 & 77.91 & 60.93 & 33.48 & 16.12 & 18.81 \\
                        &                           & PGD  & 86.38 & 84.69 & 81.19 & 73.72 & 52.24 & 11.89 & 0.19  & 0.0   \\
\midrule
\multirow{6}{*}{DefEAT} & \multirow{2}{*}{$2^{-8}$} & FGSM & 86.38 & 85.40 & 81.98 & 72.73 & 53.28 & 30.34 & 18.13 & 19.65 \\
                        &                           & PGD  & 84.52 & 82.07 & 76.51 & 63.71 & 33.87 & 1.76  & 0.0   & 0.0   \\
                        \cmidrule(lr){2-11}
                        & \multirow{2}{*}{$2^{-7}$} & FGSM & 86.69 & 85.70 & 83.05 & 74.35 & 56.36 & 32.04 & 18.88 & 21.40 \\
                        &                           & PGD  & 85.11 & 82.87 & 78.00 & 66.54 & 38.70 & 3.12  & 0.0   & 0.0   \\
                        \cmidrule(lr){2-11}
                        & \multirow{2}{*}{$2^{-6}$} & FGSM & 84.14 & 83.57 & 81.03 & 74.63 & 57.67 & 28.11 & 13.17 & 19.38 \\
                        &                           & PGD  & 82.96 & 81.35 & 77.37 & 69.52 & 50.70 & 10.05 & 0.2   & 0.0   \\
\midrule
\multirow{6}{*}{LTD}    & \multirow{2}{*}{$2^{-8}$} &  FGSM & 84.94 & 83.87 & 81.15 & 72.80 & 55.45 & 33.06 & 18.44 & 17.48 \\
                        &                           &  PGD  & 83.13 & 80.68 & 75.53 & 63.52 & 34.81 & 2.64  & 0.0   & 0.0   \\
                        \cmidrule(lr){2-11}
                        & \multirow{2}{*}{$2^{-7}$} &  FGSM & 85.48 & 84.67 & 82.24 & 74.10 & 57.41 & 35.53 & 17.57 & 17.50 \\
                        &                           &  PGD  & 83.88 & 81.88 & 77.00 & 65.38 & 28.57 & 3.95  & 0.0   & 0.0   \\
                        \cmidrule(lr){2-11}
                        & \multirow{2}{*}{$2^{-6}$} &  FGSM & 85.06 & 84.28 & 82.21 & 75.77 & 60.79 & 34.39 & 14.34 & 15.42 \\
                        &                           &  PGD  & 83.91 & 82.00 & 78.33 & 69.82 & 49.95 & 11.63 & 0.21  & 0.0   \\
\midrule
\multirow{6}{*}{AWP}    & \multirow{2}{*}{$2^{-8}$} &  FGSM & 85.11 & 83.90 & 80.68 & 71.28 & 53.78 & 33.21 & 20.86 & 19.61 \\
                        &                           &  PGD  & 83.34 & 80.34 & 75.08 & 61.53 & 30.50 & 1.89  & 0.03  & 0.0   \\
                        \cmidrule(lr){2-11}
                        & \multirow{2}{*}{$2^{-7}$} &  FGSM & 85.50 & 84.68 & 81.75 & 73.56 & 57.19 & 35.50 & 20.63 & 20.16 \\
                        &                           &  PGD  & 83.94 & 81.62 & 76.34 & 65.57 & 34.16 & 2.89  & 0.02  & 0.0   \\
                        \cmidrule(lr){2-11}
                        & \multirow{2}{*}{$2^{-6}$} &  FGSM & 83.87 & 83.19 & 81.08 & 74.97 & 61.49 & 35.13 & 16.12 & 18.69 \\
                        &                           &  PGD  & 83.00 & 81.42 & 78.13 & 71.50 & 54.95 & 14.78 & 0.41  & 0.0   \\
\midrule
\multirow{6}{*}{TRADES} & \multirow{2}{*}{$2^{-8}$} &  FGSM & 84.74 & 83.51 & 70.58 & 70.50 & 54.23 & 36.53 & 23.97 & 23.46 \\
                        &                           &  PGD  & 82.62 & 79.72 & 72.81 & 57.30 & 24.21 & 1.21  & 0.01  & 0.0   \\
                        \cmidrule(lr){2-11}
                        & \multirow{2}{*}{$2^{-7}$} &  FGSM & 85.00 & 84.02 & 80.57 & 71.61 & 55.94 & 25.67 & 22.31 & 22.51 \\
                        &                           &  PGD  & 82.96 & 80.31 & 73.75 & 58.91 & 26.22 & 1.31  & 0.02  & 0.0   \\
                        \cmidrule(lr){2-11}
                        & \multirow{2}{*}{$2^{-6}$} &  FGSM & 85.05 & 84.19 & 81.40 & 75.07 & 60.24 & 36.99 & 21.81 & 21.21 \\
                        &                           &  PGD  & 83.23 & 81.33 & 75.81 & 64.56 & 36.76 & 3.37  & 0.02  & 0.0   \\
\bottomrule
\end{tabular}
\caption{The robust accuracy against FGSM and PGD attacks among different radii of $\epsilon$-ball on CIFAR10 dataset.}
\label{table:full_acc_FGSM_PGD_CIFAR10}

\end{table*}

\begin{table*}[t]
\centering
\begin{tabular}{ccc|cccccccccccccc}
\toprule
\multirow{2}{*}{Method} & \multirow{2}{*}{$c^r$} & \multirow{2}{*}{Attack} & \multicolumn{8}{c}{$\epsilon$} \\
                        &                        &                         & $\frac{1}{255}$ & $\frac{2}{255}$ & $\frac{4}{255}$ & $\frac{8}{255}$ & $\frac{16}{255}$ & $\frac{32}{255}$ & $\frac{64}{255}$ & $\frac{96}{255}$ \\
\midrule
\multirow{6}{*}{EffAug} & \multirow{2}{*}{$2^{-8}$} & FGSM & 68.02 & 65.96 & 60.36 & 49.65 & 33.90 & 17.39 & 7.11  & 5.77 \\
                        &                           & PGD  & 64.92 & 61.04 & 52.82 & 39.37 & 17.53 & 1.81  & 0.0   & 0.0   \\
                        \cmidrule(lr){2-11}
                        & \multirow{2}{*}{$2^{-7}$} & FGSM & 68.41 & 66.56 & 61.84 & 51.83 & 36.59 & 18.34 & 7.02  & 6.70  \\
                        &                           & PGD  & 65.66 & 62.02 & 54.58 & 41.55 & 19.70 & 2.26  & 0.0   & 0.0   \\
                        \cmidrule(lr){2-11}
                        & \multirow{2}{*}{$2^{-6}$} & FGSM & 67.84 & 67.25 & 64.65 & 57.97 & 44.23 & 22.20 & 8.26  & 8.85  \\
                        &                           & PGD  & 66.38 & 64.20 & 59.63 & 50.89 & 33.98 & 7.27  & 0.17  & 0.0   \\
\midrule
\multirow{6}{*}{DKLD}   & \multirow{2}{*}{$2^{-8}$} & FGSM & 63.47 & 61.94 & 58.15 & 48.86 & 34.39 & 17.73 & 6.26  & 3.66  \\
                        &                           & PGD  & 60.71 & 57.14 & 50.44 & 38.14 & 17.38 & 1.97  & 0.0   & 0.0   \\
                        \cmidrule(lr){2-11}
                        & \multirow{2}{*}{$2^{-7}$} & FGSM & 63.55 & 62.31 & 58.65 & 50.37 & 36.49 & 18.69 & 6.36  & 4.29  \\
                        &                           & PGD  & 61.06 & 57.84 & 51.51 & 39.99 & 19.71 & 2.37  & 0.0   & 0.0   \\
                        \cmidrule(lr){2-11}
                        & \multirow{2}{*}{$2^{-6}$} & FGSM & 63.26 & 62.77 & 60.41 & 55.18 & 43.86 & 21.17 & 5.50  & 4.97  \\
                        &                           & PGD  & 61.67 & 59.62 & 55.83 & 48.06 & 33.51 & 7.55  & 0.17  & 0.0   \\
\midrule
\multirow{6}{*}{DefEAT} & \multirow{2}{*}{$2^{-8}$} & FGSM & 65.57 & 64.36 & 59.88 & 49.38 & 32.48 & 15.38 & 5.50  & 3.30  \\
                        &                           & PGD  & 62.39 & 58.96 & 51.85 & 38.59 & 17.18 & 1.45  & 0.0   & 0.0   \\
                        \cmidrule(lr){2-11}
                        & \multirow{2}{*}{$2^{-7}$} & FGSM & 65.97 & 64.96 & 60.67 & 51.44 & 34.84 & 16.36 & 5.33  & 3.96  \\
                        &                           & PGD  & 62.94 & 59.83 & 52.98 & 41.05 & 20.07 & 2.02  & 0.0   & 0.0   \\
                        \cmidrule(lr){2-11}
                        & \multirow{2}{*}{$2^{-6}$} & FGSM & 64.69 & 63.58 & 61.31 & 54.47 & 40.06 & 16.91 & 4.64  & 4.46  \\
                        &                           & PGD  & 62.85 & 60.50 & 56.29 & 47.54 & 30.76 & 5.84  & 0.09  & 0.0   \\
\midrule
\multirow{6}{*}{LTD}    & \multirow{2}{*}{$2^{-8}$} & FGSM & 63.59 & 62.35 & 58.10 & 48.86 & 33.11 & 16.78 & 5.92  & 3.20  \\
                        &                           & PGD  & 61.06 & 57.65 & 50.70 & 38.21 & 18.21 & 1.98  & 0.0   & 0.0   \\
                        \cmidrule(lr){2-11}
                        & \multirow{2}{*}{$2^{-7}$} & FGSM & 64.05 & 62.87 & 59.02 & 50.16 & 34.90 & 17.27 & 5.54  & 3.11  \\
                        &                           & PGD  & 61.51 & 58.32 & 51.84 & 39.89 & 20.32 & 2.26  & 0.0   & 0.0   \\
                        \cmidrule(lr){2-11}
                        & \multirow{2}{*}{$2^{-6}$} & FGSM & 63.62 & 62.98 & 60.96 & 54.96 & 41.51 & 19.78 & 4.69  & 3.12  \\
                        &                           & PGD  & 61.90 & 59.68 & 55.23 & 46.49 & 29.13 & 5.74  & 0.05  & 0.0   \\
\midrule
\multirow{6}{*}{AWP}    & \multirow{2}{*}{$2^{-8}$} & FGSM & 59.77 & 58.06 & 54.21 & 45.54 & 30.98 & 16.69 & 6.49  & 3.97  \\
                        &                           & PGD  & 56.72 & 52.92 & 46.53 & 34.90 & 16.03 & 2.11  & 0.0   & 0.0   \\
                        \cmidrule(lr){2-11}
                        & \multirow{2}{*}{$2^{-7}$} & FGSM & 60.00 & 58.52 & 54.93 & 46.65 & 32.66 & 17.42 & 6.04  & 3.60  \\
                        &                           & PGD  & 57.19 & 53.75 & 47.46 & 36.22 & 17.48 & 2.58  & 0.0   & 0.0   \\
                        \cmidrule(lr){2-11}
                        & \multirow{2}{*}{$2^{-6}$} & FGSM & 60.20 & 59.71 & 57.47 & 52.05 & 39.78 & 21.10 & 5.70  & 3.90  \\
                        &                           & PGD  & 58.19 & 55.66 & 50.91 & 42.37 & 25.53 & 5.68  & 0.09  & 0.0   \\
\bottomrule
\end{tabular}
\caption{The robust accuracy against FGSM and PGD attacks among different radii of $\epsilon$-ball on CIFAR100 dataset.}
\label{table:full_acc_FGSM_PGD_CIFAR100}
\end{table*}

\section{Quantitative Analysis of Empirical Lipschitz Constant}
\begin{figure*}[t]
    \centering
    \begin{subfigure}{.48\textwidth}
        \centering
        \includegraphics[width=.95\linewidth]{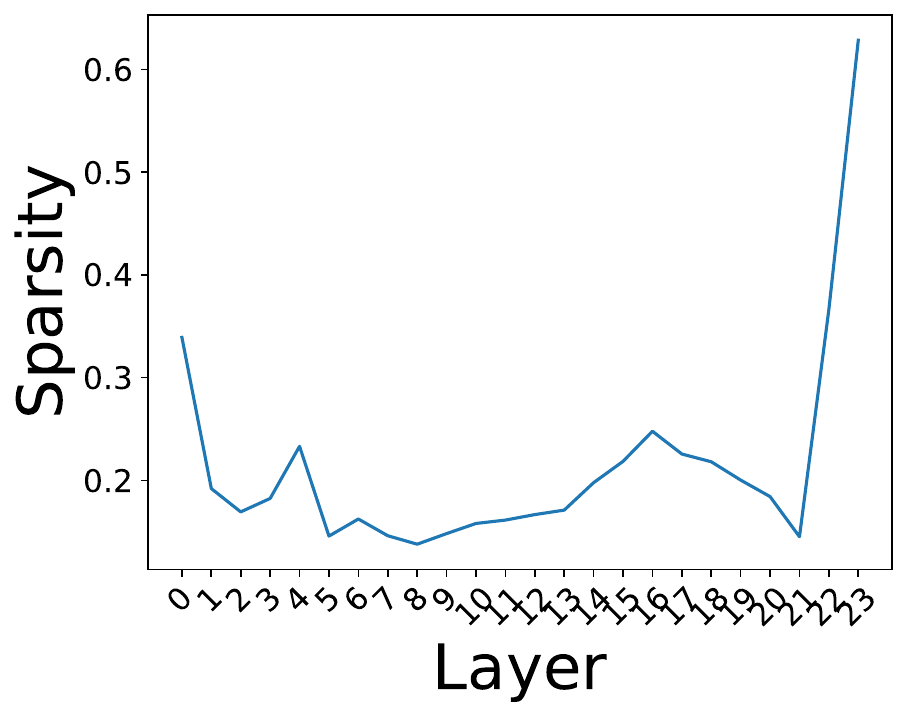}
        \caption{FC1 layers}
    \end{subfigure}
    \begin{subfigure}{0.48\textwidth}
        \centering
        \includegraphics[width=.95\linewidth]{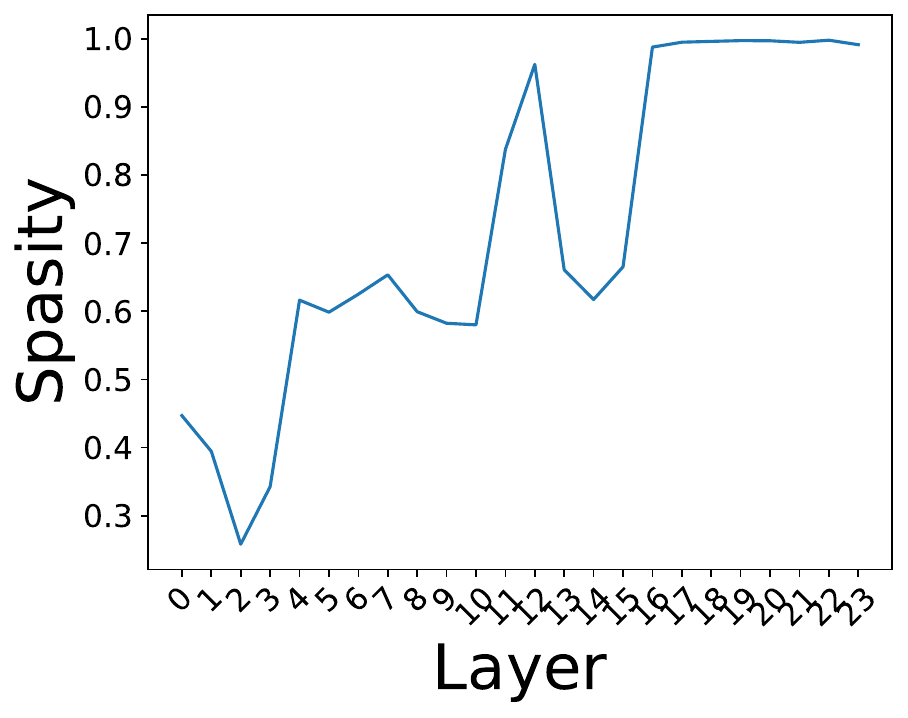}
        \caption{FC2 layers}
    \end{subfigure}
    \caption{The average proportion of pruned activations depends on the location of each linear layer. FC1 and FC2 refer to the first and second fully connected layers in MLP blocks, respectively. }
    \label{fig:act_rate}
\end{figure*}
Our proposed approach, however, is data-driven, relying on observed data to automatically determine the appropriate parameters for constructing the forged function. The sparsity of the forged vectors is a crucial factor influencing the magnitude of the Lipschitz constant for the corresponding layers, although it is not the only factor. Figure \ref{fig:act_rate} illustrates the average proportion of pruned activations, which varies depending on the location of each linear layer. FC1 and FC2 refer to the first and second fully connected layers in the MLP blocks, respectively. As shown, the proportion of pruned activations is approximately 20\% for the FC1 layers, except for the last layer. This suggests that the forged function alone cannot significantly reduce the magnitude of the Lipschitz constant.

On the other hand, the pruned rates for the FC2 layer range from 30\% to 95\%, depending on the layer's location. However, we emphasize that a higher pruning rate does not necessarily lead to a substantial reduction in the empirical Lipschitz constant. This is because, although the Lipschitz constant estimated from most of the data may decline, the empirical Lipschitz constant is based on the worst-case scenario across the entire observed dataset. Experimental results show that, after applying our method, the eigenvalue of the worst-case scenario is approximately 95\% of the original eigenvalue.

\end{document}